\documentclass[10pt,conference]{IEEEtran}
\IEEEoverridecommandlockouts
\usepackage{cite}
\usepackage{amsmath}
\usepackage{amsfonts}
\usepackage{bm}  
\usepackage{graphicx}
\usepackage{hyperref}
\usepackage{amsthm} 
\usepackage{authblk} 
\usepackage{graphicx, amsmath, amsthm, amssymb, subcaption, url, cite, array, amsthm,booktabs,xcolor}
\newtheorem{theorem}{Theorem}

\usepackage{lipsum} 
\newtheorem{assumption}{Assumption}

\newtheorem{proposition}{Proposition}

\IEEEaftertitletext{\vspace{-1.2\baselineskip}} 
\begin{document}
\title{\huge  Spiking Personalized Federated Learning for Brain-Computer Interface-Enabled Immersive Communication
}
\author[1]{Chen Shang}
\author[1]{Dinh Thai Hoang}
\author[1]{Diep N. Nguyen}
\author[2]{Jiadong Yu}
\affil[1]{School of Electrical and Data Engineering, University of Technology Sydney, Australia}
\affil[2]{Thrust of Internet of Things, The Hong Kong University of Science and Technology (Guangzhou), China}
\maketitle
\vspace{-20pt}

\begin{abstract}
This work proposes a novel immersive communication framework that leverages brain-computer interface (BCI) to acquire brain signals for inferring user-centric states (e.g., intention and perception-related discomfort), thereby enabling more personalized and robust immersive adaptation under strong individual variability. Specifically, we develop a personalized federated learning (PFL) model to analyze and process the collected brain signals, which not only accommodates neurodiverse brain-signal data but also prevents the leakage of sensitive brain-signal information.
To address the energy bottleneck of continual on-device learning and inference on energy-limited immersive terminals (e.g., head-mounted display), we further embed spiking neural networks (SNNs) into the PFL. By exploiting sparse, event-driven spike computation, the SNN-enabled PFL reduces the computation and energy cost of training and inference while maintaining competitive personalization performance. Experiments on real brain-signal dataset demonstrate that our method achieves the best overall identification accuracy while reducing inference energy by 6.46$\times$ compared with conventional artificial neural network-based personalized baselines.

\end{abstract}

\vspace{-5pt}
\section{Introduction}
\vspace{-5pt}
The paradigm shift toward 6G is accelerating the realization of immersive communication as a core pillar of next-generation wireless networks~\cite{itu_m2160}. By supporting ultra-high data rates, extremely low latency, and highly reliable connectivity, immersive communication system is expected to unlock a new class of applications such as extended/virtual reality (XR/VR), holographic telepresence, and real-time multi-user collaboration, where users can seamlessly interact with digital content and remote environments with near-natural perception and responsiveness. To fully deliver such immersive experiences, the system typically relies on collaborative sensing across heterogeneous devices, such as head-mounted display (HMD), wearables, and surrounding IoT nodes. By combining these complementary modalities, the network can construct a richer and more reliable view of the session context (e.g., human-centric metaverse~\cite{10496440} and human digital twins~\cite{shang2024biologically}), which is essential for responsive interaction, adaptive rendering, and intelligent service orchestration.

However, realizing such multi-sensor-driven immersive intelligence is challenging. 
First, the collected signals are often heterogeneous in formats and statistics due to diverse hardware platforms and vendor-specific sensing pipelines, resulting in inconsistent sampling rates, imperfect synchronization, varying noise levels, and even missing measurements. This heterogeneity significantly increases the service-side burden, which must perform real-time alignment, fusion, and learning over high-dimensional and time-varying streams under strict latency and resource constraints. 
Meanwhile, immersive terminals (e.g., AR/VR headsets) are inherently resource-constrained. The practical form-factor requirements (e.g., small size and low weight) translate into limited battery capacity and thermal headroom, which constrain sustained on-device computation and wireless transmission, and also limit the available footprint for sensors~\cite{10706833}.
Therefore, it is highly desirable to reduce the reliance on numerous heterogeneous sensing modalities and sensor types while preserving immersive service quality, thereby alleviating service-side fusion complexity and terminal-side burdens, and improving scalability and robustness in practical deployments

More importantly, ``immersive experiences’’ are not solely determined by external physical-context sensing. They also depend critically on users' internal perceptual and physiological states in the real world.
For example, cybersickness can be triggered by perception-related factors (e.g., sensory conflicts) and user-specific susceptibility, which are difficult to infer or mitigate from external sensing alone~\cite{10706833}. Therefore, immersive communication systems should incorporate user-state estimation from physiological and neural signals, which provide observable correlates of user-centric responses (e.g., discomfort, workload, attention, and engagement). This enables real-time, closed-loop adaptation of rendering and interaction to maintain stable quality of experience across individuals and applications~\cite{shang2024biologically,10496440}.

To tackle the above challenges and fulfill the above requirements, this work proposes a novel brain-computer interface (BCI)-driven immersive communication framework. As an information-rich modality that directly reflects users' perception, intention, and neuro-physiological responses, brain activity provides a complementary sensing channel beyond conventional environment and motion-centric sensors (e.g., inertial measurement unit (IMU) and camera)~\cite{10496440}. By non-invasively acquiring brain signals such as electroencephalography (EEG) via BCI devices, the immersive system can infer user-centric states (e.g., intention and perception-related discomfort) in real time, enabling more personalized adaptation and improved immersion robustness under strong individual variability~\cite{10496440}. To fully exploit this potential under practical on-device constraints, we develop an energy-efficient personalized federated learning (PFL) algorithm for brain-signal processing, where spiking neural networks (SNNs) are integrated to enable sparse, event-driven computation and reduce the energy cost of training and inference while maintaining competitive personalization performance under heterogeneous user data (i.e., neurodiversity).
We further theoretically prove that the sparse spiking activity in SNNs can mitigate data drift in PFL by reducing gradient dissimilarity across users, thereby improving training stability and personalization performance under heterogenous brain-signal distributions. The experimental results conducted on real brain dataset (i.e., EEG signals~\cite{goldberger2000physiobank}) show that, under the same network architecture, the proposed SNN-enabled PFL achieves the highest identification accuracy while reducing energy consumption by 6.46$\times$ compared with the traditional artificial neural network (ANN)-based personalized baseline.

\vspace{-5pt}
\section{Preliminaries and System Overview}
\subsection{Brain-Computer Interface and Brain Signal Process}
BCI enables direct brain signal acquisition and interpretation, which provides a unique sensing modality for human-centric 6G immersive communication systems~\cite{10496440}.
By continuously decoding cognitive states and user intents (e.g., attention level, mental workload, or imagined actions), brain signal processing can support proactive and context-aware adaptation, such as attention-aware rendering, interaction assistance, and personalized digital-twin updating~\cite{10496440,shang2024biologically}.
In this sense, BCI-driven perception complements conventional wearable or IoT sensing by injecting neural-level information into the immersive loop, thereby improving the realism and responsiveness of immersive services~\cite{10496440,10706833}.

In this work, we focus on EEG as a representative brain signal for BCI-enabled immersive applications, since it can be acquired non-invasively~\cite{10496440,10706833} and has been actively explored in emerging BCI-HMD prototypes~\cite{cognixion}.
Traditional EEG-based BCI systems typically rely on a hand-crafted processing pipeline, including signal preprocessing (e.g., filtering and artifact suppression), feature extraction in the time-frequency domains (e.g., spatial filtering), and shallow classifiers (e.g., SVM)~\cite{10496440}.
Such pipelines heavily rely on hand-crafted priors and carefully tuned feature extractors, making their performance sensitive to preprocessing choices and often requiring non-trivial subject and session-specific calibration.
In contrast, deep learning provides a promising alternative by learning task-relevant representations end-to-end from raw EEG and reducing reliance on manual feature design and heuristic preprocessing, which can improve robustness to complex spatio-temporal patterns and facilitate more adaptive, personalized inference for immersive interactions\cite{10496440,10706833}.

\subsection{System Overview}
\begin{figure}[t]
    \centering
    \includegraphics[width=0.37\textwidth]{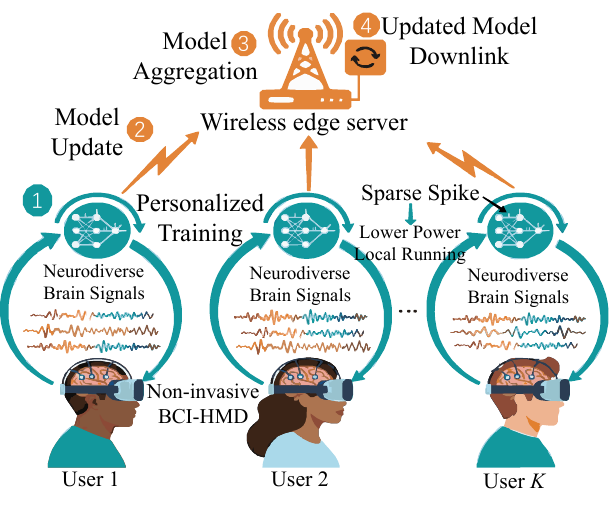}
    \caption{Illustration of the proposed BCI-enabled immersive communication framework supported by SNN-driven PFL. A wireless edge server (WES) hosts immersive applications and delivers VR/AR experiences to $K$ users. Each user is equipped with an integrated BCI-HMD that enables simultaneous immersive interaction and brain-signal acquisition. The acquired brain signals are processed locally by an SNN with sparse, event-driven spiking activity, enabling energy-efficient inference and continual personalization.}
    \label{fig:SNNforHDT}
\end{figure}


Benefiting from the powerful representation-learning ability of deep learning, we consider to use deep learning models for EEG signal processing and analysis.
Nevertheless, directly deploying deep learning model for BCI-enabled immersive systems remains non-trivial.
First, brain signals are highly privacy-sensitive, making continuous uploading for centralized training or inference undesirable. Although federated learning (FL) can keep raw signals on-device and only exchange model updates~\cite{fl}, the repeated local training, frequent uplink communication, and always-on inference with dense ANN computation can still impose heavy energy overhead on resource-constrained immersive terminals~\cite{9583900}.

On the other hand, brain signals exhibit strong \emph{neurodiversity} (as illustrated in Fig.~\ref{fig:SNNforHDT}), making a traditional \emph{one-size-fits-all} solution difficult to generalize and often leading to unstable training under heterogeneous  data.
For example, EEG is typically recorded from multiple scalp channels distributed over different brain regions, and the resulting multi-channel waveforms are highly sensitive to subject-specific neuroanatomy, electrode placement and contact impedance, and session-dependent factors such as fatigue, attention, and emotion~\cite{10496440}. Consequently, neurodiversity arises not only between heterogeneous users, but also within the same user across channels and time, yielding non-stationary and uneven signal statistics that challenge robust personalization and continual learning.
These challenges collectively motivate a \emph{privacy-preserving}, \emph{personalization-aware}, and \emph{energy-efficient} edge-learning framework.

As shown in Fig.~\ref{fig:SNNforHDT}, the proposed BCI-enabled immersive communication framework consists of
a set of $K$ user devices $\mathcal{K}=\{1,2,\ldots,K\}$ and a wireless edge server (WES).
Each user device (i.e., a BCI-enabled HMD) collects local brain-signal data and performs SNN-driven on-device training and inference,
while the WES coordinates the PFL framework over wireless links. By integrating BCI with PFL, the system can continuously update personalized brain-signal analysis models across users without collecting raw EEG signals, thereby enabling real-time estimation of user-centric states (e.g., intention and perception-related discomfort) for adaptive rendering and interaction.
Meanwhile, the on-device SNN-enabled learning framework further reduces the compute and energy cost of continual training and inference, making the proposed framework suitable for long-term, always-on immersive applications.
Overall, the proposed design (i) alleviates the service-side burden of multimodal data processing and enriches immersive-system sensing by introducing a BCI that directly captures user-centric neuro-physiological states that are difficult to infer from conventional environment and motion centric sensors, and (ii) enables privacy-preserving and energy-efficient personalization via SNN-driven PFL, supporting continual adaptation under neurodiverse user data.

\vspace{-5pt}
\section{Proposed Learning Algorithm}\label{sec3}
\vspace{-5pt}
In this section, we first present the energy-efficient SNN enabled personalized FL. We then analyze the convergence and advantages of the proposed algorithm.

\vspace{-5pt}
\subsection{Personalized Federated Learning}\label{Section:PFL}

In immersive communication systems, brain signals (e.g., EEG) exhibit inherent neurodiversity across users, which often results in non-IID local datasets $\{\mathcal{D}_k\}_{k\in\mathcal{K}}$.
In such a case, learning a \emph{single} global model via standard FL (e.g., FedAvg~\cite{fl}) may yield suboptimal and uneven inference quality across users.
To better accommodate user heterogeneity, we introduce PFL~\cite{9743558}, where each user-$k$ maintains a personalized model $\boldsymbol{w}_k$ while collaboratively learning a shared global reference model $\boldsymbol{w}$. Accordingly, the global optimization objective can be expressed as:
\begin{equation}
\min_{\boldsymbol{w},\{\boldsymbol{w}_k\}} 
\sum_{k=1}^{K}\frac{|\mathcal{D}_k|}{\mathcal{D}}
\left(
F_k(\boldsymbol{w}_k) + \frac{\mu}{2}\left\|\boldsymbol{w}_k-\boldsymbol{w}\right\|_2^2
\right),
\label{eq:pfl_obj}
\end{equation}
where $\mathcal{D}=\sum_{k=1}^{K}{\mathcal{D}_k}$ denote the entire dataset, $F_k(\boldsymbol{w}_k)=\mathcal{L}(\boldsymbol{w}_k,\mathcal{D}_k)$ is the predetermined loss function, and $\mu>0$ controls the degree of personalization.
Moreover, at each global round $r$, each user performs local optimization to update its personalized model by approximately minimizing the following regularized objective~\cite{9743558}:
\begin{equation}
\boldsymbol{w}_k^{(r+1)}
\approx
\arg\min_{\boldsymbol{w}_k}
\left(
F_k(\boldsymbol{w}_k) + \frac{\mu}{2}\left\|\boldsymbol{w}_k-\boldsymbol{w}^{(r)}\right\|_2^2
\right),
\label{eq:pfl_local}
\end{equation}
and uploads $\boldsymbol{w}_k^{(r+1)}$ to the WES.
The WES then updates the global reference model via weighted aggregation:  
\begin{equation}
    \boldsymbol{w}^{(r+1)}=\sum_{k=1}^{K}\frac{|\mathcal{D}_k|}{\mathcal{D}}\,\boldsymbol{w}^{(r)}_k.
    \label{eq:fedavg}
    \end{equation}
This PFL mechanism enables each user to retain a personalized model for reliable brain-state decoding while still leveraging cross-user collaboration through the shared global reference model.

\vspace{-5pt}
\subsection{Energy-efficient Personalized FL-enabled by SNN}\label{algorithm}

To support continuous brain signal processing on energy-constrained user devices, we adopt SNNs as the model backbone in PFL.
Different from conventional ANN relying on continuous activations and multiply-and-accumulate (MAC) operations~\cite{9583900}, SNNs propagate discrete binary spikes in an event-driven manner~\cite{9583900,panda2020toward}.
This characteristic enables training and inference dominated by accumulate (AC) operations, thereby reducing computational complexity and energy consumption while naturally capturing temporal dependencies in EEG signals. The details of the SNN are elaborated as follows.

\subsubsection{SNN Framework}
We adopt the leaky integrate-and-fire (LIF) neuron model~\cite{10636728,shang2025energy} to construct SNN.
Let $\mathcal{U}_j(t)$ denote the membrane potential of neuron-$j$ at the SNN time step $t$. The dynamic of LIF neuron model is given by:
\begin{equation}
\mathcal{U}_j(t+1) = (1-\lambda)\mathcal{U}_j(t) + \lambda \sum_{i=1}^M W_{i,j} \mathcal{I}_i(t),
\label{LIF}
\end{equation}
where $\lambda$ is the leakage coefficient, $W_{i,j}$ is the synaptic weight, and $\mathcal{I}_i(t)$ is the input from presynaptic neuron $i$.
When $\mathcal{U}_j(t)$ exceeds a firing threshold $\mathcal{U}_{\text{th}}$, neuron-$j$ emits a spike $\mathcal{S}_j(t)=1$ and resets its potential, otherwise $\mathcal{S}_j(t)=0$. Based on the LIF neuron dynamics, we construct the SNN and describe its training procedure as follows.

\subsubsection{Training of SNN-based Personalized FL}
Let $\tilde{F}_k(\boldsymbol{w}_k)=\tilde{\mathcal{L}}\big(\boldsymbol{w},\mathcal{D}_k\big)$ denote the loss of SNN-enabled PFL. Training the SNN-based neural network requires back-propagation through time (BPTT)~\cite{9583900} due to the temporal recursion in \eqref{LIF}.
Therefore, the gradient of the loss $\tilde{\mathcal{L}}$ with respect to the synaptic weight $W_{i,j}$ can be written as:
\begin{equation}
\begin{aligned}
	&\frac{\partial \tilde{\mathcal{L}}}{\partial W_{i,j}}=\sum_{t=1}^T{\frac{\partial \tilde{\mathcal{L}}}{\partial \mathcal{U} _j\left( t \right)}\frac{\partial \mathcal{U} _j\left( t \right)}{\partial W_{i,j}}}\\
	&=\sum_{t=1}^{T-1}{\frac{\partial \tilde{\mathcal{L}}}{\partial \mathcal{U} _j\left( t \right)}\frac{\partial \mathcal{U} _j\left( t \right)}{\partial W_{i,j}}} +\frac{\partial \tilde{\mathcal{L}}}{\partial \mathcal{U} _j\left( T \right)}\frac{\partial \mathcal{U} _j\left( T \right)}{\partial W_{i,j}}\\
	&=\sum_{t=1}^{T-1}{\left( \frac{\partial \tilde{\mathcal{L}}}{\partial \mathcal{S} _j\left( t \right)}\frac{\partial \mathcal{S} _j\left( t \right)}{\partial \mathcal{U} _j(t)}+\frac{\partial \tilde{\mathcal{L}}}{\partial \mathcal{U} _j(t+1)}\frac{\partial \mathcal{U} _j(t+1)}{\partial \mathcal{U} _j(t)} \right)} \\
	&~\quad \times \frac{\partial \mathcal{U} _j(t)}{\partial W_{i,j}}+\frac{\partial \tilde{\mathcal{L}}}{\partial \mathcal{S} _j\left( T \right)}\frac{\partial \mathcal{S} _j\left( T \right)}{\partial \mathcal{U} _j(T)}\frac{\partial \mathcal{U} _j(T)}{\partial W_{i,j}}\\
\end{aligned},
\label{eq.bp}
\end{equation}
where $T$ is the total number of SNN time steps.
However, the spiking function is non-differentiable due to its binary output, i.e., $\mathcal{S}_j(t)=0$ or $\mathcal{S}_j(t)=1$.
To enable effective gradient-based optimization, we adopt a surrogate gradient method~\cite{10636728,shang2025energy,fang2021incorporating}.
Specifically, during back-propagation, we replace the non-differentiable spike function with a smooth surrogate $\varphi(\cdot)$, which is given by:
\begin{equation}
\varphi(\mathcal{S}) = \frac{1}{\pi} \arctan\left(\frac{\pi \eta}{2}\mathcal{S}\right) + \frac{1}{2}, ~~\varphi ^{\prime}\left( \mathcal{S} \right) =
\frac{\eta}{2}\frac{1}{1+\left( \frac{\pi \eta}{2}\mathcal{S} \right) ^2},
\label{ada}
\end{equation}
where $\eta$ is a tunable smoothing parameter and $\varphi'(\mathcal{S})$ denotes the derivative of $\varphi(\mathcal{S})$.
As a result, surrogate-based training enables efficient learning of SNN-based personalized EEG signal processing within the PFL procedure. It is worth noting that the surrogate function in~\eqref{ada} is only used to approximate gradients in back-propagation, while the forward dynamics still follow the original LIF model in \eqref{LIF}.

\subsection{Convergence Analysis and Benefits of SNN-based PFL}\label{sec:theory}
We present the algorithm's convergence analysis in this subsection. Specifically, although the spike generation in the LIF dynamics~\eqref{LIF} is non-differentiable, surrogate-gradient learning enables gradient-based training by replacing the spike derivative with a smooth surrogate in back-propagation~\cite{8891809}, i.e., \eqref{ada}. 
Consequently, training SNNs with surrogate gradients is equivalent to performing stochastic gradient-based optimization on an implicitly defined differentiable objective, i.e., $\tilde{F}_k$.
This observation allows us to analyze the convergence of SNN-based PFL by applying standard PFL theory~\cite{10.5555/3495724.3497520} to the surrogate objective, without requiring differentiability of the original spike model. 
In the following, we first present a convergence guarantee (with respect to the surrogate objective) and then discuss why SNN can improve robustness under neurodiverse brain signals while reducing energy consumption. To this end, we first introduce the following assumption has been widely utilized in FL works~\cite{10.5555/3495724.3497520,10229017}.

\vspace{-5pt}
\begin{assumption}[Smooth surrogate objectives]\label{assump:smooth}
For each user-$k$, $\tilde{F}_k(\cdot)$ is lower bounded and has $L$-Lipschitz continuous gradients, i.e., 
$\big\|\nabla \tilde{F}_k(\boldsymbol{w})-\nabla \tilde{F}_k(\boldsymbol{w}')\big\|\le L\big\|\boldsymbol{w}-\boldsymbol{w}'\big\|$ for all $\boldsymbol{w},\boldsymbol{w}'\in\mathbb{R}^{d}$.
Here, $d$ is the number of trainable parameters (i.e., the dimension of the model parameter vector $\boldsymbol{w}$), and $L>0$ is the smoothness constant that upper bounds how fast the gradient can change with respect to $\boldsymbol{w}$.
\end{assumption}
\vspace{-5pt}
\begin{assumption}[Bounded-variance stochastic gradients]\label{assump:variance}
For each user-$k$, let $\boldsymbol{g}_k(\boldsymbol{w}_k;\xi)$ be the stochastic gradient computed at $\boldsymbol{w}_k$ using a random mini-batch $\xi$ sampled from the local dataset $\mathcal{D}_k$.
Then, for any $\boldsymbol{w}_k\in\mathbb{R}^{d}$, it holds that $\mathbb{E}_{\xi}\!\left[\boldsymbol{g}_k(\boldsymbol{w}_k;\xi)\right]=\nabla \tilde{F}_k(\boldsymbol{w}_k)$ and $\mathbb{E}_{\xi}\!\left\|\boldsymbol{g}_k(\boldsymbol{w}_k;\xi)-\nabla \tilde{F}_k(\boldsymbol{w}_k)\right\|_2^2\le \sigma^2$.
Here, $\mathbb{E}_{\xi}[\cdot]$ denotes expectation over mini-batch sampling, and $\sigma^2$ is an upper bound on the gradient estimation variance induced by stochastic mini-batches.
\end{assumption}
\vspace{-5pt}

\subsubsection{Convergence analysis}\label{subsec:surrogate_obj}
Recall that the SNN is trained via surrogate gradients, where the spike derivative is replaced by the smooth surrogate in \eqref{ada}.
Accordingly, the actual optimization performed by each user-$k$ corresponds to minimizing a differentiable \emph{surrogate loss} $\tilde{F}_k(\boldsymbol{w}_k)$.
Therefore, the regularized PFL objective in \eqref{eq:pfl_obj} is written in the surrogate form as:
\begin{equation}
\min_{\boldsymbol{w},\{\boldsymbol{w}_k\}}
\sum_{k=1}^{K} p_k
\left(
\tilde{F}_k(\boldsymbol{w}_k) + \frac{\mu}{2}\|\boldsymbol{w}_k-\boldsymbol{w}\|_2^2
\right),
\label{eq:pfl_surrogate_obj}
\end{equation}
where $p_k \triangleq |\mathcal{D}_k|/|\mathcal{D}|$. At global round $r$, user-$k$ approximately solves the proximal subproblem:
\begin{equation}
\boldsymbol{w}_k^{(r+1)} \approx 
\arg\min_{\boldsymbol{w}_k}\;
\tilde{F}_k(\boldsymbol{w}_k) + \frac{\mu}{2}\big\|\boldsymbol{w}_k-\boldsymbol{w}^{(r)}\big\|_2^2
\label{eq:local_prox_sur}
\end{equation}
by performing $E$ local mini-batch updates via surrogate BPTT, where $E$ denotes the number of local update steps executed by each user per global round.
Specifically, for a mini-batch $\xi$ sampled from $\mathcal{D}_k$, a single local update takes the stochastic gradients (SGD), which is given by:
\begin{equation}
\boldsymbol{w}_k \leftarrow 
\boldsymbol{w}_k - \alpha
\Big(
\boldsymbol{g}_k(\boldsymbol{w}_k;\xi) + \mu(\boldsymbol{w}_k-\boldsymbol{w}^{(r)})
\Big),
\label{eq:prox_sgd_step}
\end{equation}
where $\alpha>0$ is the step size.
After collecting $\{\boldsymbol{w}_k^{(r+1)}\}_{k=1}^K$, the WES updates the global model by weighted averaging,
$\boldsymbol{w}^{(r+1)}=\sum_{k=1}^{K}p_k\,\boldsymbol{w}_k^{(r+1)}$.

We now establish a convergence guarantee with respect to the surrogate objective. Specifically, we introduce the following theorem.
\vspace{-5pt}
\begin{theorem}\label{thm:convergence_sur}
Consider the above PFL procedure that, at each global round $r$, performs $E$ local updates of the form~\eqref{eq:prox_sgd_step} to approximately solve~\eqref{eq:local_prox_sur},
followed by weighted aggregation to update $\boldsymbol{w}^{(r)}$.
Under standard step-size conditions required by proximal-PFL analyses~\cite{10.5555/3495724.3497520},
the sequence $\{\boldsymbol{w}^{(r)}\}$ satisfies:
\begin{equation}
\min_{0\le r \le R-1}\;
\mathbb{E}\big[\|\nabla \tilde{F}_{\mu}(\boldsymbol{w}^{(r)})\|_2^2\big]
\;\longrightarrow\; 0
\quad \text{as } R\to\infty,
\label{eq:stationary_sur}
\end{equation}
where $\tilde{F}_{\mu}(\boldsymbol{w})$ is the surrogate envelope (i.e., the surrogate Moreau envelope~\cite{10.5555/3495724.3497520}) associated with \eqref{eq:pfl_surrogate_obj}, defined as
\begin{equation}
\tilde{F}_{\mu}(\boldsymbol{w})
\triangleq
\sum_{k=1}^{K} p_k
\min_{\boldsymbol{w}_k}\;
\tilde{F}_k(\boldsymbol{w}_k) + \frac{\mu}{2}\|\boldsymbol{w}_k-\boldsymbol{w}\|_2^2.
\label{eq:surrogate_envelope_def}
\end{equation}
Hence, the proposed SNN-based PFL converges to a stationary point of the surrogate envelope.
\end{theorem}

\vspace{-7pt}
\begin{proof}
By surrogate-gradient training, each user optimizes the differentiable surrogate loss $\tilde{F}_k$.
Assumption~\ref{assump:smooth} ensures that $\tilde{F}_k$ has $L$-Lipschitz continuous gradients (possibly nonconvex),
and Assumption~\ref{assump:variance} ensures that the mini-batch gradient $\boldsymbol{g}_k(\boldsymbol{w}_k;\xi)$ is an unbiased estimator of $\nabla \tilde{F}_k(\boldsymbol{w}_k)$ with bounded variance.
Therefore, the local update \eqref{eq:prox_sgd_step} is exactly a stochastic proximal-gradient step for the smooth proximal subproblem \eqref{eq:local_prox_sur}.
Under these regularity conditions and standard step-size conditions, the convergence theory of pFedMe applies directly~\cite{10.5555/3495724.3497520}.
Consequently, the proposed SNN-driven PFL achieves stationarity on the surrogate envelope $\tilde{F}_{\mu}(\boldsymbol{w})$.
\end{proof}
\vspace{-5pt}

\subsubsection{SNN sparsity for drift mitigation under neurodiversity}\label{subsec:benefits_drift}
As aforementioned, a key challenge in BCI-enabled immersive communication systems is the intrinsic neurodiversity of EEG data, which arises not only \emph{across users} but also \emph{within the same user} across heterogeneous scalp channels and over time. Such heterogeneous and non-stationary data distributions can lead to user drift in federated optimization.
Specifically, the gradient dissimilarity can be used to quantify such drift, which is defined as:
\begin{equation}
\Delta(\boldsymbol{w})\triangleq \sum_{k=1}^{K} p_k
\big\|\nabla \tilde{F}_k(\boldsymbol{w}) - \nabla \tilde{F}(\boldsymbol{w})\big\|_2^2,
\label{eq:drift_def}
\end{equation}
where $\tilde{F}(\boldsymbol{w})\triangleq \sum_{k=1}^{K} p_k\,\tilde{F}_k(\boldsymbol{w})$.
Smaller $\Delta(\boldsymbol{w})$ generally implies more stable aggregation and improved training performance under heterogeneity,
and motivates drift-correction methods in FL~\cite{karimireddy2020scaffold}. Note that $\Delta(\boldsymbol{w})$ captures only data-driven gradient heterogeneity at a common $\boldsymbol{w}$, $\mu$ merely provides proximal anchoring and is thus omitted.

By introducing SNNs, the local training dynamics become event-driven and typically exhibit sparse spiking activity.
This sparsity, together with the bounded surrogate derivative in \eqref{ada}, further constrains the effective back-propagation signals and thus can yield a smaller upper bound on $\|\nabla \tilde{F}_k(\boldsymbol{w})\|_2$.
Such a bound can be further translated into a tighter upper bound on the drift metric $\Delta(\boldsymbol{w})$ in \eqref{eq:drift_def}.
Following the standard surrogate-gradient SNN training literature~\cite{8891809}, we assume a bounded surrogate derivative and sparse spiking activity to control gradient magnitudes and capture event-driven computation.
\vspace{-5pt}
\begin{assumption}[Bounded surrogate derivative and sparse activity]\label{assump:snn_bounds}
The surrogate derivative is bounded as $|\varphi'(\cdot)|\le c_\varphi$, where $c_\varphi>0$ controls the maximum magnitude of the surrogate gradient and prevents exploding gradients during BPTT.
Moreover, the spiking activity is sparse: letting $\boldsymbol{s}^{(\ell)}(t)\in\{0,1\}^{d_\ell}$ denote the spike vector at layer $\ell$ and SNN time step $t$, there exists $\rho\in(0,1]$ such that $\mathbb{E}\|\boldsymbol{s}^{(\ell)}(t)\|_0 \le \rho d_\ell$ for all $\ell,t$.
Here, $\|\cdot\|_0$ counts the number of nonzero entries (i.e., the number of spikes), $d_\ell$ is the number of neurons in layer $\ell$, and $\rho$ is the firing sparsity level: smaller $\rho$ indicates sparser spiking activity.
\end{assumption}
\vspace{-10pt}
\begin{proposition}[Sparsity-induced drift upper bound]\label{prop:drift}
Suppose that for all users $k$ and all iterates, the gradient norms satisfy
$\|\nabla \tilde{F}_k(\boldsymbol{w})\|_2 \le G(\rho)$ for some non-decreasing function $G(\rho)$.
Then, for any $\boldsymbol{w}$:
\begin{equation}
\Delta(\boldsymbol{w}) \le 4\,G(\rho)^2.
\label{eq:drift_bound_general}
\end{equation}
Moreover, under Assumption~\ref{assump:snn_bounds} and standard boundedness conditions on the network parameters and signals, the gradient magnitude scales with the spiking activity. Therefore, one can upper bound $G(\rho)$ as $G(\rho)\le C\sqrt{\rho}$ for constant $C>0$, which yields:
\begin{equation}
\Delta(\boldsymbol{w}) \le 4 C^2 \rho.
\label{eq:drift_bound_rho}
\vspace{-5pt}
\end{equation}
\end{proposition}

\begin{proof}
We first prove \eqref{eq:drift_bound_general}.
Using $\|a-b\|^2 \le 2\|a\|^2 + 2\|b\|^2$ and Jensen's inequality, \eqref{eq:drift_def} can be recast as:
\begin{align}
\Delta(\boldsymbol{w})
&= \sum_{k} p_k \|\nabla \tilde{F}_k(\boldsymbol{w}) - \nabla \tilde{F}(\boldsymbol{w})\|_2^2 \nonumber\\
&\le 2\sum_k p_k \|\nabla \tilde{F}_k(\boldsymbol{w})\|_2^2 + 2\|\nabla \tilde{F}(\boldsymbol{w})\|_2^2 \nonumber\\
&= 2\sum_k p_k \|\nabla \tilde{F}_k(\boldsymbol{w})\|_2^2 + 2\Big\|\sum_k p_k \nabla \tilde{F}_k(\boldsymbol{w})\Big\|_2^2 \nonumber\\
&\le 2\sum_k p_k G(\rho)^2 + 2\sum_k p_k \|\nabla \tilde{F}_k(\boldsymbol{w})\|_2^2 \nonumber\\
&\le 4G(\rho)^2,
\end{align}
which proves \eqref{eq:drift_bound_general}.

To obtain the $\rho$-scaling, we bound the surrogate gradient norm using Assumption~\ref{assump:snn_bounds}.
In surrogate-gradient BPTT, the gradient with respect to synaptic weights is a sum of products between (i) a back-propagated error term and
(ii) presynaptic spike activities, i.e., \eqref{eq.bp}.
By the bounded surrogate derivative $|\varphi'(\cdot)|\le c_\varphi$ and standard boundedness of network weights signals,
the back-propagated error terms remain uniformly bounded.
Meanwhile, the presynaptic spike vectors are sparse:
$\mathbb{E}\|\boldsymbol{s}^{(\ell)}(t)\|_2 \le \sqrt{\mathbb{E}\|\boldsymbol{s}^{(\ell)}(t)\|_0}
\le \sqrt{\rho d_\ell}$.
Therefore, the overall surrogate gradient satisfies $\|\nabla \tilde{F}_k(\boldsymbol{w})\|_2 \le C\sqrt{\rho}$ for some constant $C>0$.
Substituting $G(\rho)\le C\sqrt{\rho}$ into \eqref{eq:drift_bound_general} yields \eqref{eq:drift_bound_rho}, which completes the proof.
\vspace{-8pt}
\end{proof}

\noindent\textbf{Discussion.}
The convergence analysis in Theorem~\ref{thm:convergence_sur} guarantees that the surrogate-based proximal-PFL procedure converges to a stationary point of the surrogate envelope.
Proposition~\ref{prop:drift} further indicates that sparse spiking activity (small $\rho$) can tighten an upper bound on gradient dissimilarity and hence mitigate user drift.
Since user drift is a major factor behind instability and slow convergence under non-IID data~\cite{karimireddy2020scaffold},
this provides theoretical support for why SNN-based PFL can potentially improve robustness in neurodiverse BCI-enabled immersive communication systems.


\subsubsection{Energy efficiency enabled by sparse spiking activity}\label{subsec:benefits_energy}
Beyond robustness, spiking sparsity directly translates into reduced computation. Together with the AC operation described in Section~\ref{algorithm}, sparse firing (i.e., small $\rho$) reduces the number of effective synaptic events per time step, which scales with $\rho$. This directly translates into lower computation and energy consumption on resource-limited BCI wearables.
In our experiments, the average spiking activity is $\rho \approx 0.12$, meaning that only about $12\%$ of neurons emit spikes at a given time step, which substantially reduces the number of effective synaptic events compared with dense ANN activations.
Therefore, SNN-based PFL offers a principled way to jointly address neurodiversity-induced training difficulty and long-term energy constraints in 6G immersive communication systems.

\vspace{-5pt}
\section{Performance Evaluations} \label{PE}
\vspace{-5pt}
\subsection{Experimental Settings}\label{simulation setting}
\subsubsection{BCI data statement}
\vspace{-2pt}

We use the public EEG dataset~\cite{goldberger2000physiobank} to simulate brain-signal processing in the BCI-enabled immersive communication framework. The dataset contains EEG recordings from 109 participants under controlled motor-imagery paradigms, where each subject is instructed to perform or imagine a specific action according to visual cues on a screen. We select participants 1, 3, and 7 as federated clients and formulate EEG-based user-intent identification via 4-class motor imagery (left/right fist, both fists, both feet) for immersive interaction, more detailed can be found in ~\cite{goldberger2000physiobank,shang2024biologically}.


\textit{2) Baselines and Algorithm Parameters:}
We compare the proposed method, denoted as PFLSNN, with three baselines: (i) PFL with ANN (PFLANN), (ii) FL with SNN (FLSNN), and (iii) FL with ANN (FLANN). For a fair comparison, all methods adopt the same CNN architecture (i.e., identical layer widths and input/output dimensions), where the only differences lie in the learning paradigm (FL vs.\ PFL) and the backbone implementation (SNN vs.\ ANN). In particular, the SNN model follows the same convolutional feature extractor and classifier dimensionality as the ANN counterpart, while using LIF neurons and temporal spike accumulation. Unless otherwise specified, the learning hyper-parameters are set to: learning rate $0.01$, batch size $64$, local epochs $E=2$, and proximal regularization coefficient $\mu=10^{-5}$. Moreover, the SNN parameters $\{T,\lambda,\eta\}$ are set to $\{6,0.5,2\}$, respectively.

\vspace{-2pt}
\subsection{Experimental Results}
\vspace{-2pt}
\subsubsection{Algorithm Performance on Brain Signal Identification Accuracy}
As shown in Fig.~\ref{fig:accuracy}, the proposed PFLSNN exhibits the most favorable learning performance and the highest identification accuracy, reaching 87.53\% at round 50, which consistently stays above the other baselines throughout training. In comparison, PFLANN converges to 81.90\% (round 48), while FLSNN and FLANN plateau at 68.78\% (round 48) and 62.12\% (round 48), respectively.  
We further quantify the gains brought by personalization and by the spiking architecture at convergence. At the final round, switching from FL to PFL yields a large accuracy improvement for both backbones. PFLSNN improves over FLSNN by +18.75\%, and PFLANN improves over FLANN by +19.78\%, confirming that user-specific adaptation is crucial in neurodiverse EEG decoding. Moreover, SNNs consistently outperform ANNs under the same federated paradigm: under PFL, PFLSNN exceeds PFLANN by +5.63\%, and under FL, FLSNN exceeds FLANN by +6.66\%. 
These results highlight two trends: (i) \emph{personalization} mitigates the neurodiversity-induced non-IID distribution shift by allowing each user to adapt its local decoder, and (ii) the \emph{spiking backbone} captures temporal EEG dynamics with sparse event-driven activations, yielding more stable updates in heterogeneous training.

\subsubsection{Energy Consumption}
\begin{figure}[t]
    \centering
    \includegraphics[width=0.60\linewidth]{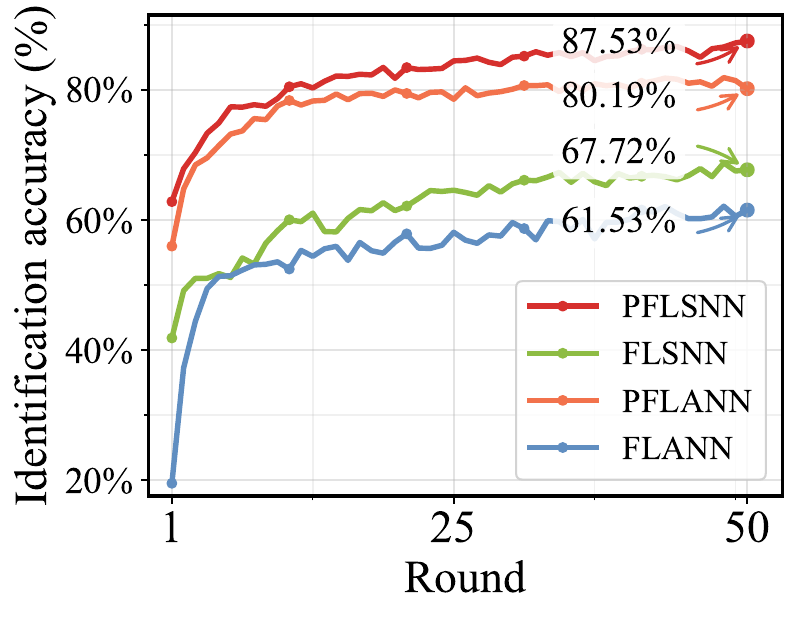}
    \caption{Brain signal identification accuracy vs. training round for
various methods.}
    \label{fig:accuracy}
\end{figure}
\begin{figure}[t]
    \centering
    \includegraphics[width=0.60\linewidth]{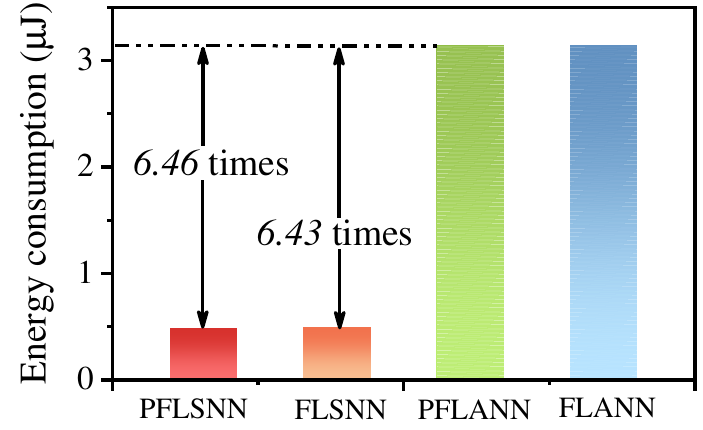}
    \caption{ Energy consumption of various methods during model inference.}
    \label{fig:energy}
    \vspace{-5pt}
\end{figure}
We then evaluate the inference energy consumption of different algorithms. Following prior works~\cite{shang2025energy,shang2024biologically,panda2020toward,9583900}, we quantify the compute energy based on synaptic operations and spiking activity. Specifically, we measure the average firing rate (spiking activity) $\rho$ of each LIF layer in the SNN, which equals $0.106$, $0.063$, $0.058$, and $0.253$, respectively. As illustrated in Fig.~\ref{fig:energy}, the results confirm the energy advantage of spiking computation: under the same network architecture, PFLSNN reduces inference energy by about $6.46\times$ compared with its ANN-based personalized counterpart (i.e., PFLANN). Taken together, experiments on real EEG data validate that the proposed SNN-enabled personalized FL not only achieves higher identification accuracy, but also substantially lowers compute energy, supporting its suitability for continuous on-device adaptation in immersive communication systems.


\section{Conclusion}\label{conclusion}
In this work, we have proposed a BCI-enabled personalized immersive communication framework that leverages brain signals to infer user-centric states, enabling more personalized and robust immersive adaptation under strong individual variability. To accommodate neurodiverse brain-signal data while preventing leakage of sensitive information, we have formulated brain-signal learning and inference under a PFL paradigm. To address the energy bottleneck of continual on-device learning and inference, we have embedded SNNs into the PFL procedure to exploit sparse, event-driven computation. Experiments on real brain-signal dataset have shown that the proposed method achieves the best overall identification accuracy while reducing inference energy by $6.46\times$ over ANN-based personalized baselines, highlighting its potential for long-term, always-on immersive communication systems.

\vspace{-8pt}
\bibliographystyle{IEEEtran}
\bibliography{bibRef}
\vspace{-15pt}
\end{document}